\documentclass{colt2017}
\usepackage{times}
\usepackage{comment}
\usepackage{graphicx}
\usepackage{enumitem}
\graphicspath{{images/} }
\makeatletter
\def\temp{dvips.def}
\ifx\Gin@driver\temp
\def\Ginclude@graphics#1{\def\temp{#1}---image \expandafter\strip@prefix\meaning\temp---}
\fi
\makeatother

\usepackage{epsf}
\usepackage{amsmath}
\usepackage{amssymb}
\usepackage{color}

\newcommand{\ignore}[1]{}

\ignore{
\newtheorem{theorem}{Theorem}

\newtheorem{lemma}[theorem]{Lemma}

}

\newcommand{\FF}{{\cal F}}

\newcommand{\R}{\mathbb{R}}

\newcommand{\N}{\mathbb{N}}

\newcommand{\De}{{\rm De}}
\newcommand{\As}{{\rm As}}

\newcommand{\DE}{{\rm DE}}
\newcommand{\AS}{{\rm AS}}
\newcommand{\lca}{{\rm lca}}

\newcommand{\OPT}{{\rm OPT}}
\newcommand{\Ray}{{\rm Ray}}

\newcommand{\G}{{\cal G}}
\newcommand{\I}{{\cal I}}

\newcommand{\imply}{{\Rightarrow}}

\newcounter{comment}

\newcommand{\qed}{\hfill\ensuremath{\square}}
\renewcommand{\qed}{}
\definecolor{mygray}{gray}{0.5}

\newcommand{\DanaC}[1]{{\color{red} (Comment \arabic{comment}: #1)}\stepcounter{comment}}
\newcommand{\DanaF}[2]{{\color{blue} #1 {{\color{mygray} (old: #2)}}}}
\newcommand{\DanaD}[1]{{\color{blue} #1}}
\newcommand{\DanaFD}[2]{#1}

\renewcommand{\DanaD}[1]{#1}

\coltauthor{\Name{Nader H. Bshouty} \Email{bshouty@cs.technion.ac.il}\\
 \addr Technion
 \AND
 \Name{Dana Drachsler-Cohen} \Email{ddana@cs.technion.ac.il}\\
 \addr Technion
 \AND
 \Name{Martin Vechev} \Email{martin.vechev@inf.ethz.ch}\\
 \addr ETH Zurich
 \AND
 \Name{Eran Yahav} \Email{yahave@cs.technion.ac.il}\\
 \addr Technion
 }
\begin{document}

\title{Learning Disjunctions of Predicates}


%
\maketitle
\begin{abstract} Let $\FF$ be a set of boolean functions. We present an algorithm for learning $\FF_\vee:=\{\vee_{f\in S}f\ |\ S\subseteq \FF\}$ from membership queries. Our algorithm asks at most $|\FF|\cdot \OPT(\FF_\vee)$ membership queries where $\OPT(\FF_\vee)$ is the minimum worst case number of membership queries for learning $\FF_\vee$. When $\FF$ is a set of halfspaces over a constant dimension space or a set of variable inequalities, our algorithm runs in polynomial time. 

The problem we address has practical importance in the field of program synthesis, where the goal is to synthesize a program that meets some requirements. Program synthesis has become popular especially in settings aiming to help end users. In such settings, the requirements are not provided upfront and the synthesizer can only learn them by posing membership queries to the end user. Our work enables such synthesizers to learn the exact requirements while bounding the number of membership queries.
\end{abstract}

\section{Introduction}
Learning from membership queries \citep{A87} has flourished due to its many applications in group testing \citep{DH00,DH06}, blood testing \citep{D43}, chemical leak testing, chemical reactions \citep{AC08}, electrical short detection, codes, multi-access channel communications \citep{BG07}, molecular biology, VLSI testing, AIDS screening, whole-genome shotgun sequencing \citep{ABKRS04}, DNA physical mapping \citep{GK98} and game theory \citep{P02}. For a list of many other applications, see~\cite{DH00,ND00,BGV05,DH06,Ci13,BG07}. Many of the new applications present new models and new problems\DanaFD{. One of these is programming by example (PBE), a popular setting of program synthesis \citep{Polozov:2015,Barowy:2015,FlashExtract:14,Gulwani:2011,BitManipulation:2010}. PBE has gained popularity because it enables end users to describe their intent to a program synthesizer via the intuitive means of input--output examples. The common setting of PBE is to synthesize a program based on a typically small set of user-provided examples, which are often an under-specification of the target program \citep{Polozov:2015,Barowy:2015,FlashExtract:14,Gulwani:2011}. As a result, the synthesized program is not guaranteed to fully capture the user's intent. Another (less popular) PBE approach is to limit the program space to a small (finite) set of programs and ask the user membership queries while there are non-equivalent programs in the search space \citep{BitManipulation:2010}. A natural question is whether one can do better than the latter approach without sacrificing the correctness guaranteed by the former approach. In this paper, we answer this question for a class of specifications (i.e., formulas) that captures a wide range of programs.
}{and in many of those applications, such as ..., the function being learned can be a composition of any predicates in a predefined set~\cite{}.}

\DanaFD{We study the problem of learning a disjunctive (or dually, a conjunctive) formula describing the user intent through membership queries. To capture a wide range of program specifications, the formulas are over arbitrary, predefined predicates.}{
In this paper, we study the learnability of the conjunction or disjunction of predicates.}
In our setting, the end user is the teacher that can answer membership queries. This work enables PBE synthesizers to guarantee to the user that they have synthesized the correct program, while bounding the number of membership queries they pose, thus reducing the burden on the user.

\DanaFD{More formally, let}{Let} $\FF$ be a set of predicates (i.e., boolean functions). Our goal is to learn the class $\FF_\vee$ of any disjunction of predicates in $\FF$. We present a learning algorithm SPEX, which learns any function in $\FF_\vee$ with polynomially many queries. \DanaFD{We then show that given}{Assuming} some computational complexity conditions on the set of predicates, SPEX \DanaFD{}{also} runs in polynomial time.

\DanaFD{We demonstrate the above on}{We apply the above to} two classes. The first is the class of disjunctions (or conjunctions, whose learning is the dual problem) \DanaFD{over}{of} any set $H$ of halfspaces over a constant dimension. For this class, we show that SPEX runs in polynomial time\DanaFD{}{(for $H_\vee$ and $H_\wedge$)}. In particular, this \DanaFD{shows that learning}{learns} any convex polytope over a constant dimension when the sides are from a given set $H$ can be done in polynomial time. For the case where the dimension is not constant, we show that learning this class implies P=NP.
\DanaD{We note that there are other applications for learning halfspaces; for example,~\cite{Hegedus:1995,Zolotykh95,Abboud:1999,Abasi2014}}.

The second \DanaFD{class we consider}{set} is conjunctions over $\FF$, where $\FF$ is the set of variable \DanaFD{inequalities}{inequality}, i.e., predicates of the form $[x_i>x_j]$ over $n$ variables. If the set is acyclic ($\wedge \FF\not=0$), we show that learning can be done in polynomial time. If the set is cyclic ($\wedge \FF =0$), we show that learning is equivalent to the problem of enumerating all the maximal acyclic subgraphs of a directed graph\DanaFD{, which}{. This problem} is still an open problem \citep{ABC12,BCL13,W16}.

The second class has practical importance because it consists of formulas that can be used to describe time-series charts. Time charts are used in many domains including financial analysis \citep{EncPatterns}, medicine \citep{Chuah2007}, and seismology \citep{MoralesEsteban}. Experts use these charts to predict important events (e.g., trend changes in a stock price) by looking for \emph{patterns} in the charts. A lot of research has focused on common patterns and many platforms enable these experts to write a program that upon detecting a specific pattern alerts the user (e.g., some platforms for finance analysts are \href{http://www.metaquotes.net/en/metatrader5}{MetaTrader},
\href{http://www.metastock.com/}{MetaStock},
\href{http://amibroker.com/}{Amibroker}). Unfortunately, writing programs is a complex task for these experts, as they are not programmers.
To help such experts, we integrated SPEX in a synthesizer that interacts with a user to learn a pattern (i.e., a conjunctive formula over variable inequalities). SPEX enables the synthesizer to guarantee that the synthesized program captures the user intent, while interacting with him only through membership queries that are visualized in charts.

The paper is organized as follows. Section~\ref{sec:model} describes the model and class we consider. Section~\ref{sec:defres} provides the main definitions we require for the algorithm. Section~\ref{sec:alg} presents the SPEX algorithm, discusses its complexity, and describes conditions under which SPEX is polynomial. Sections~\ref{sec:halfs} and \ref{sec:vars} discuss the two classes we consider: halfspaces and variable inequalities. Finally, Section~\ref{sec:synthesis} shows the practical application of SPEX in program synthesis.

\section{The Model and\DanaFD{}{the} Class}\label{sec:model}

\DanaFD{Let $\FF$ be}{Given} a \DanaD{finite} set of boolean functions over a domain $X$ \DanaD{(possibly infinite)}. We consider the class of functions $\FF_\vee:=\{\vee_{f\in S}f\ |\ S\subseteq \FF\}$. Our model
\DanaFD{assumes}{Given} a {\it teacher} that has a {\it target function} $F \in \FF_\vee$ and a {\it learner} that knows $\FF$ but not the target function. The teacher can answer {\it membership queries} for the target function -- that is, \DanaFD{given}{when receiving} $x\in X$ (from the learner), the teacher returns $F(x)$.  The goal of the learner (the learning algorithm) is to find the target function with a minimum number of membership queries.

\paragraph{Notations} Following are a few notations used throughout the paper.
\DanaFD{}{Let}$\OPT(\FF_\vee)$ denotes the minimum worst case number of membership queries \DanaFD{required}{that is needed} to learn a function $F$ in $\FF_\vee$.
\DanaFD{}{Define}Given $F\in \FF_\vee$, we denote by $S(F)$ the set that consists of all the functions in~$F$.
\DanaD{Formally, we define $S(F)=S$, where $S$ is the unique subset of $\mathcal{F}$ such that $F=\bigvee_{f\in S}f$.}
For example, $S(f_1 \vee f_2)=\{f_1,f_2\}$. \DanaFD{From this, it immediately follows that}{Then,} $F_1\equiv F_2$ if and only if $S(F_1)=S(F_2)$. For a set of functions $S\subseteq \FF$, we denote $\vee S:=\vee_{f\in S}f$.
Lastly, $[{\cal S}(x)]$ denotes the boolean value of a logical statement. Namely, given a statement ${\cal S}(x):X\to \{T,F\}$ with a free variable $x$, \DanaFD{its}{the} boolean function $[{\cal S}(x)]:X \to \{0,1\}$ is defined as $[{\cal S}(x)]= 1$ if ${\cal S}(x)=T$, and $[{\cal S}(x)]=0$ otherwise.
For example, $[x \ge 2]=1$ if and only if the interpretation of $x$ is greater than $2$.

\section{Definitions and Preliminary Results}\label{sec:defres}
In this section, we provide the definitions used in this paper and show preliminary results. We begin by defining an equivalence relation over the set of disjunctions and defining the representatives of the equivalence classes. Thereafter, we define a partial order over the disjunctions and related notions (descendant, ascendant, and lowest/greatest common descendant/ascendant). We complete this section with the notion of a witness, which is central to our algorithm.

\subsection{\DanaFD{An Equivalence Relation Over $\FF_\vee$}{Partial Order \DanaF{over}of $\FF_\vee$}}\label{sec:defs}
In this section, we present an equivalence relation over $\FF_\vee$ and define the representatives of the equivalence classes. This enables us in later sections to focus on the representative elements from $\FF_\vee$.
Let $\FF$ be a set of boolean functions over the domain $X$. \DanaFD{The}{Define the} equivalence relation $=$ over $\FF_\vee$ \DanaFD{is defined as follows:}{where} two disjunctions $F_1,F_2\in \FF_\vee$ are equivalent ($F_1= F_2$) if $F_1$ is logically equal to $F_2$. In other words, they represent the same function \DanaFD{(from $X$ to $\{0,1\}$)}{$X\to \{0,1\}$}. We \DanaFD{}{will} write $F_1\equiv F_2$ to denote that $F_1$ and $F_2$ are identical\DanaFD{; that}{. That} is, they have the same representation. For example, consider $f_1,f_2:\{0,1\}\to \{0,1\}$ where $f_1(x)=1$ and $f_2(x)=x$. Then, $f_1\vee f_2 = f_1$ but $f_1 \vee f_2 \not\equiv f_1$.

We denote by $\FF_\vee^*$ the set of equivalence classes of $=$ and write each equivalence class as $[F]$, where $F\in\FF_\vee$. Notice that if $[F_1]=[F_2]$, then $[F_1\vee F_2]=[F_1]=[F_2]$. Therefore, for every $[F]$, we can choose the {\it representative element} to be $G_F:=\vee_{F'\in S}F'$ where $S\subseteq \FF$ is the maximum size set that satisfies $\vee S=F$. We denote by $G(\FF_\vee)$ the set of all representative elements. Accordingly, $G(\FF_\vee)=\{G_F\ |\ F\in\FF_\vee\}$.
\DanaFD{As an example, consider}{Before we proceed we give an example. Consider} the set $\FF$ consisting of four functions $f_{11},f_{12},f_{21},f_{22}:\{1,2\}^2\to \{0,1\}$ where $f_{ij}(x_1,x_2)=[x_i\ge j]$. There are $2^4=16$ elements in $\Ray^2_2:=\FF_\vee$ and five representative functions in $G(\FF_\vee)$: $G(\FF_\vee)=\{f_{11}\vee f_{12}\vee f_{21}\vee f_{22}$, $f_{12}\vee f_{22}$, $f_{12}$, $f_{22},0\}$ (where $0$ is the zero function).

The below listed facts follow\DanaFD{}{s} immediately from the above definitions:
\begin{lemma}\label{sfact} \DanaFD{Let $\FF$ be a set of boolean functions. Then,}{We have}
\begin{enumerate}[nosep,nolistsep]
\item The number of logically non-equivalent boolean functions in $\FF_\vee$ is $|G(\FF_\vee)|$.
\item For every $F\in\FF_\vee$\DanaFD{,}{we have} $G_F=F$.
\item \label{sfact3} For every $G\in G(\FF_\vee)$ and $f\in \FF\backslash S(G)$\DanaFD{,}{we have} $G\vee f\not=G$.
\item For every $F\in\FF_\vee$\DanaFD{,}{we have} $\vee S(F)\equiv F$.
\item If $G_1,G_2\in G(\FF_\vee)$, then $G_1=G_2$ if and only if $G_1\equiv G_2$.
\end{enumerate}
\end{lemma}

\subsection{A Partial Order Over $\FF_\vee$}\label{sec:32}
In this section, we define a partial order over $\FF_\vee$ and present related definitions.
The partial order, denoted by $\imply$, is defined as follows:
$F_1\imply F_2$ if $F_1$ logically implies~$F_2$.
Consider the Hasse diagram $H(\FF_\vee)$ of $G(\FF_\vee)$ for this partial order. The maximum (top) element in the diagram is $G_{\max}:=~\vee_{f\in \FF}f$. The minimum (bottom) element is $G_{\min}:=\vee_{f\in \O}f$, i.e., the zero function. Figure~\ref{HasseRay22} shows an illustration of the Hasse diagram of $\Ray^2_2$ (from Section~\ref{sec:defs}). \DanaD{Figures~\ref{HasseClause32} and \ref{RAY23E} show other examples of Hasse diagrams: Figure~\ref{HasseClause32} shows the Hasse diagram of boolean variables, while Figure~\ref{RAY23E} shows an example that extends the example of $\Ray^2_2$.}

In a Hasse diagram, $G_1$ is a {\it descendant} (resp., {\it ascendent}) of $G_2$ if there is a (nonempty) downward path from $G_2$ to $G_1$ (resp., from $G_1$ to $G_2$), i.e., $G_1\imply G_2$ (resp., $G_2\imply G_1$) and $G_1\not=G_2$. $G_1$ is an {\it immediate descendant} of $G_2$ in $H(\FF_\vee)$ if $G_1\imply G_2$, $G_1\not=G_2$ and there is no $G$ such that $G\not= G_1$, $G\not=G_2$ and $G_1\imply G\imply G_2$. $G_1$ is an {\it immediate ascendant} of $G_2$ if $G_2$ is an immediate descendant of $G_1$.
We now show (all proofs for this section appear in Appendix~\ref{sec3proofs}):
\begin{lemma}\label{fact} Let $G_1$ be an immediate descendant of $G_2$ and $F\in\FF_\vee$.
If $G_1\imply F\imply G_2$, then $G_1=F$ or $G_2=F$.
\end{lemma}

We denote by $\De(G)$ and $\As(G)$ the sets of all the immediate descendants and immediate ascendants of \DanaFD{}{and neighbours of}$G$, respectively. We further denote by $\DE(G)$ and $\AS(G)$ the sets of all $G$'s descendants and ascendants, respectively.
For $G_1$ and $G_2$, we define their {\it lowest common ascendent} (resp., greatest common descendant) $G=\lca(G_1,G_2)$ (resp., $G=\gcd(G_1,G_2)$) to be the boolean function $G\in G(\FF_\vee)$ -- that is, the minimum (resp., maximum) element in $\AS(G_1)\cap \AS(G_2)$ (resp., $\DE(G_1)\cap \DE(G_2)$).
\DanaFD{}{Therefore, we can show:}Therefore, we can show Lemma~\ref{trivial}. Lemma~\ref{trivial} abbreviates ($G_1\imply G$ and $G_2\imply G$) to $G_1,G_2\imply G$ and ($G\imply G_1$ and $G\imply G_2$) to $G\imply G_1,G_2$.
\begin{lemma}\label{trivial}
Let $G_1,G_2\in G(\FF_\vee)$ and $F\in \FF_\vee$.
\begin{enumerate}[nosep,nolistsep]
\item If $G_1,G_2\imply F\imply \lca(G_1,G_2)$, then $F=\lca(G_1,G_2)$.
\item If $ \gcd(G_1,G_2)\imply F\imply G_1,G_2$, then $F=\gcd(G_1,G_2)$.
\end{enumerate}
\end{lemma}

\DanaFD{Lemma~\ref{trivial} leads us to Lemma~\ref{lca}:}{We now prove:}
\begin{lemma}\label{lca} Let $G_1,G_2\in G(\FF_\vee)$. Then, $\lca(G_1,G_2)=G_1\vee G_2$.

In particular, if $G_1,G_2$ are two distinct immediate descendants of $G$, then $G_1\vee G_2=G$.
\end{lemma}
Note that this does not imply that $S(G_1\vee G_2)=S(G_1)\cup S(G_2)=S(\lca(G_1,G_2))$. In particular, $G_1\vee G_2$ is not necessarily in $G(\FF_\vee)$; see, for example, Figure~\ref{HasseClause32} (right).

Lemma~\ref{GG} follows from the fact that if $G_1$ is a descendant of $G_2$, then $G_1\imply G_2$, and therefore, $G_1\vee G_2=G_2$.

\begin{lemma}\label{GG}
If $G_1$ is a descendant of $G_2$, then $S(G_1)\subsetneq S(G_2)$.
\end{lemma}

\DanaFD{Lemma~\ref{GG} enables us to show the following.}{We now show}
\begin{lemma}\label{gcd} Let $G_1,G_2\in G(\FF_\vee)$. Then, $S(G_1)\cap S(G_2)=S(\gcd(G_1,G_2))$.

In particular, if $G_1,G_2\in G(\FF_\vee)$, then $\vee(S(G_1)\cap S(G_2))\in G(\FF_\vee)$.

Also, if $G_1,G_2$ are two distinct immediate ascendants of~$G$, then $S(G_1)\cap S(G_2)=S(G)$.
\end{lemma}
Note that this does not imply that $G_1\wedge G_2=\gcd(G_1,G_2)$; see, for example, Figure~\ref{HasseClause32} (right).

\subsection{Witnesses}
Finally, we define the term \emph{witness}.
Let $G_1$ and $G_2$ be elements in $G(\FF_\vee)$. An element \DanaFD{$a\in X$}{$a$} is a {\it witness} for $G_1$ and $G_2$ if $G_1(a)\not= G_2(a)$. We now show two central lemmas.

\begin{lemma}\label{wit1} Let $G_1$ be an immediate descendant of $G_2$. If $a\in X$ is a witness for $G_1$ and $G_2$, then: \begin{enumerate}[nosep,nolistsep]
\item $G_1(a)=0$ and $G_2(a)=1$.
\item For every $f\in S(G_1)$, $f(a)=0$.
\item For every $f\in S(G_2)\backslash S(G_1)$, $f(a)=1$.
\end{enumerate}
\end{lemma}
\begin{proof} Since $G_1\imply G_2$, it must be that $G_2(a)=1$ and $G_1(a)=0$. Namely, for every $f\in S(G_1)$, $f(a)=0$.
Let $f\in S(G_2)\backslash S(G_1)$. Consider $F=G_1\vee f$. By bullet {\it \ref{sfact3}} in Lemma~\ref{sfact}, $F\not=G_1$. Since $G_1\imply F\imply G_2$, by Lemma~\ref{fact}, $F=G_2$.
Therefore, $f(a)=G_1(a)\vee f(a)=F(a)=G_2(a)=1$.\qed
\end{proof}

\begin{lemma}\label{uniqwit} Let $\De(G)=\{G_1,G_2,\ldots,G_t\}$ be the set of immediate descendants of $G$. If $a$ is a witness for $G_1$ and $G$, then $a$ is not a witness for $G_i$ and $G$ for all $i>1$. That is, $G_1(a)=0$, $G(a)=1$, and $G_2(1)=\cdots=G_t(a)=1$.
\end{lemma}
\begin{proof} By Lemma~\ref{wit1}, $G(a)=1$ and $G_1(a)=0$. By Lemma~\ref{lca}, for any $G_i$, $i\ge 2$, we have $G=G_1\vee G_i$. Therefore, $1=G(a)=G_1(a)\vee G_i(a)= G_i(a)$.\qed
\end{proof}

\section{The Algorithm}\label{sec:alg}
In this section, we present our algorithm \DanaD{to learn a target disjunction over $\FF$}, called SPEX (short for \emph{specifications from examples}).
Our algorithm \DanaFD{relies on}{uses} the following results.
\begin{lemma}\label{Al01} Let $G'$ be an immediate descendant of $G$, \DanaFD{$a\in X$}{$a$} be a witness for $G$ and $G'$, and  $G''$ be a descendant of $G$.
\begin{enumerate}[nosep,nolistsep]
\item If $G''(a)$$=$$0$, $G''$ is a descendant of $G'$ or equal to $G'$. In particular, $S(G'')\subseteq S(G')$.
\item If $G''(a)$$=$$1$, $G''$ is neither a descendant of $G'$ nor equal to $G'$. In particular, $S(G'')\not\subset S(G')$.
\end{enumerate}
\end{lemma}
\begin{proof} Since $G''$ is a descendant of $G$, we have $S(G'')\subsetneq S(G)$. By Lemma~\ref{wit1}, for every $f\in S(G')$, we have $f(a)=0$, and for every $f\in S(G)\backslash S(G')$, we have $f(a)=1$. \DanaFD{Thus, if}{If} $G''(a)=0$, then no $f\in S(G)\backslash S(G')$ is in $S(G'')$ (otherwise, $G''(a)=1$). Therefore, $S(G'')\subseteq S(G')$ and $G''$ is a descendant of $G'$ or equal to $G'$. \DanaFD{Otherwise, if}{
Now suppose} $G''(a)=1$ \DanaFD{, then since}{. Since} $G'(a)=0$, for every descendant $G_0$ of $G'$ we have $G_0(a)=0$ and thus $G''$ is neither a descendant of $G'$ nor equal to $G'$.\qed
\end{proof}

Lemma~\ref{Al01} drives the operation of SPEX. To find $F$ (more precisely, $G_F$), SPEX starts from the maximal element in $G(\FF_\vee)$ and traverses downwards through the Hasse diagram. At each step, SPEX considers an element $G$, checks its witnesses against its immediate descendants, and poses a membership query for each. If $F$ and $G$ agree on the witness of $G$ and $G'$, then by Lemma~\ref{Al01}, $F$ cannot be $G'$ or its descendant, and thus these are pruned from the search space. Otherwise, if $F$ and $G'$ agree on the witness, then $F$ must be $G'$ or its descendant, and thus all other elements in $G(F_\vee)$ are pruned.
\begin{figure}[t!]
  \begin{center}
   \fbox{\fbox{\begin{minipage}{28em}
  \begin{tabbing}
  xxxx\=xxxx\=xxxx\=xxxx\= \kill
  \>\underline{{\bf Algorithm: SPEX -- The target function is $F$.}}\\ \\
  \> Learn($G\gets G_{\max}$,$T\gets \O$).\\ \\
  \> \underline{{\bf Learn}$(G,T)$}\\
  1. \>$S\gets S(G)$; Flag$=1$.\\
  2. \> For each immediate descendant $G'$ of $G$:\\
  3. \> \> If $S(G')\not\subset R$ for all $R\in T$ then:\\
  4. \> \> \> Find a witness $a$ for $G'$ and $G$.\\
  5. \> \> \> If $F(a)=0$ then: $S\gets S\cap S(G')$; Flag$=0$.\\
  6. \> \> \> If $F(a)=1$ then: $T\gets T\cup \{S(G')\}$.\\
  7. \>\> EndIf\\
  8. \> EndFor\\
  9. \> If Flag$=1$ then: Output($\vee S$)\\
  10.\>\> Else Learn$(\vee S, T)$
  \end{tabbing}\end{minipage}}}
  \end{center}
	\caption{The SPEX algorithm for learning functions in $\FF_\vee$.}
	\label{Alg0}
	\end{figure}

The SPEX algorithm is depicted in Figure~\ref{Alg0}. SPEX calls the recursive algorithm Learn, which takes a candidate $G$ and a set of subsets of $\FF$, $T$, which stores the elements already eliminated from $\FF_\vee$. Learn also relies on $S$, a set of functions over which $F$ (i.e., $G_F$) is defined (i.e., $S(G_F)\subseteq S$). During the execution, $S$ may be reduced. If not, then $\vee S=F$.
Learn begins by initializing $S$ to $S(G)$. Then, it examines the immediate descendants of $G$ whose ancestors have not been eliminated. When considering $G'$, a witness $a$ is obtained and Learn poses a membership query to learn $F(a)$. If $F(a)=0$ (recall that $G(a)=1$ since $a$ is a witness), then $G\neq F$ and $F$ is inferred to be a descendant of $G'$ and is thus over the functions in $S(G')$. Accordingly, $S$ is reduced. Otherwise, $F$ is not a descendant of $G'$, and $G'$ and its descendants are eliminated from the search space by adding $S(G')$ to $T$.
Finally, if $G$ and $F$ agreed on all the witnesses (evident by the Flag variable), then $\vee S$ is returned (since $G=F$). Otherwise, Learn is invoked on $\vee S$ and $T$.

\begin{theorem}\label{upper} If the witnesses and the descendants of any $G$ can be found in time $t$, then SPEX (Algorithm~1) learns the target function in time $t\cdot |\FF|$ and at most
$|\FF|\cdot \max_{G\in G(\FF_\vee)} |\De(G)|$ membership queries.
\end{theorem}
\sloppy
The complexity proof \DanaD{follows from the following arguments. First, every invocation of SPEX presents a membership query for every immediate descendant of $G$, and thus the number of membership queries of a single invocation is at most the maximal number of immediate descendants, $\max_{G\in G(\FF_\vee)} |\De(G)|$. Second, recursive invocations always consider a descendant of the currently inspected candidate. Thus, the recursion depth is bounded by the height of the Hasse diagram, $|\FF|$. This implies the total bound of $|\FF|\cdot \max_{G\in G(\FF_\vee)} |\De(G)|$ membership queries.}
The fact that SPEX learns the target function follows from Lemma~\ref{lem:cor}.
\begin{lemma}\label{lem:cor} \DanaFD{}{Consider the recursive procedure Learn.} Let $F$ be the target function.
If Learn returns $\vee S$, then $G_F = \vee S$ ($\star$). Otherwise, if Learn$(G_{\max},\O)$ calls Learn$(\vee S,T)$, then:
\begin{enumerate}[nosep,nolistsep]
\item $S(G_F)\subseteq S$. That is, $G_F$ is a descendant of $\vee S$ or equal to $\vee S$.
\item $S(G_F)\not\subset R$ for all $R\in T$. That is, $G_F$ is not a descendant of any $\vee R$, for $R\in T$ or equal to $\vee R$.
\end{enumerate}
\end{lemma}
\begin{proof}
The proof is by induction. Obviously, the induction hypothesis is true for $(G_{\max},\O)$.
Assume the induction hypothesis is true for $(\vee S,T)$. That is, $S(G_F)\subseteq S$ and $S(G_F)\not\subset R$ for all $R\in T$. Let $G'_1,\ldots,G'_\ell$ be all the immediate descendants of $\vee S$. If $S(G'_i)\subseteq R$ for some $R\in T$, $G'_i$ and all its descendants $G''$ satisfy $S(G'')\subseteq S(G'_i)\subseteq R$ and thus $G_F$ is not $G'_i$ or a descendant of $G'_i$.
\DanaFD{}{This explains the condition in line 3 in the Algorithm.}

Assume now that $S(G'_i)\not\subset R$ for all $R\in T$. Let $a^{(i)}$ be a witness for $\vee S$ and $G'_i$. If $F(a^{(i)})=1$, then by Lemma~\ref{Al01} $G_F$ is not a descendant of $G'_i$ and not equal to $G'_i$. This implies that $S(G_F)\not \subset S(G_i')$, which is why $S(G_i')$ is added to $T$ (Line 6 in the Algorithm). This proves bullet 2.

If $F(a^{(i)})=1$ for all $i$, then $G_F=\vee S$. This follows since by Lemma~\ref{Al01}, $F$ is not any of $\vee S$ descendants; thus by the induction hypothesis, it must be $\vee S$. This is the case when the Flag variable does not change to $0$ and the algorithm outputs $\vee S$. This proves ($\star$).

If $F(a^{(i)})=0$, then by Lemma~\ref{Al01}, $G_F$ is a descendant of $G_i'$ or equal to $G_i'$. Let $I$ be the set of all indices $i$ for which $F(a^{(i)})=0$. Then, $G_F$ is a descendant of (or equal to) all $G_i'$, $i\in I$, and therefore, $G_F$ is a descendant of or equal to $\gcd(\{G_i'\}_{i\in I})$. By Lemma~\ref{gcd}, $S(\gcd(\{G_i'\}_{i\in I}))=\cap_{i\in I}S(G_i)$. Thus, the algorithm in Line 5 takes the new $S$ to be $\cap_{i\in I}S(G_i)$. This proves bullet 1.\qed
\end{proof}

\subsection{Lower Bound}
The number of different boolean functions in $\FF_\vee$ is $|G(\FF_\vee)|$, and therefore, \DanaFD{from}{just using} the information theoretic lower bound we get:
$\OPT(\FF_\vee)\ge \lceil \log |G(\FF_\vee)|\rceil.$
We now prove the lower bound.
\begin{theorem}\label{lower} Any learning algorithm that learns $\FF_\vee$ must ask at least
$\max(\log |G(\FF_\vee)|, \max_{G\in G(\FF_\vee)} |\De(G)|)$ membership queries.
In particular, SPEX (Algorithm~1) asks at most
$|\FF|\cdot \OPT(\FF_\vee)$ membership queries.
\end{theorem}
\begin{proof} Let $G'$ be such that  $m=|\De(G')|=\max_{G\in G(\FF_\vee)} |\De(G)|$.
Let $G_1,\ldots,G_m$ be the immediate descendants of $G'$. If the target function is either $G'$ or one of its immediate descendants, then any learning algorithm must ask a membership query $a^{(i)}$ such that $G'(a^{(i)})=1$ and $G_i(a^{(i)})=0$. Without such an assignment, the algorithm cannot distinguish between $G'$ and $G_i$. By Lemma~\ref{uniqwit}, $a^{(i)}$ is a witness only to $G_i$, and therefore, we need at least $m$ membership queries.\qed
\end{proof}

\subsection{Finding All Immediate Descendants of $G$}
A missing detail in our algorithm is how to find the immediate descendants of $G$ in the Hasse diagram $H(S(G))$.
In this section, we explain how to obtain them. We first characterize the elements in $H(S(G))$ (compared to the other elements in $\FF_\vee$), which is necessary because the immediate descendants are part of $H(S(G))$. We then give a characterization of the immediate descendants (compared to other descendants), which leads to an operation that computes an immediate descendant from a descendant. We finally show how to compute descendants that lead to obtaining different immediate descendants. This completes the description of how SPEX can obtain all immediate descendants.

By the definition of a representative, for \DanaFD{every}{any} $F\in \FF_\vee$,
$G_F=\vee_{f\imply F} f.$
\DanaFD{To decide whether $F\in \FF_\vee$ is a representative, i.e., whether $F\in G(\FF_\vee)$, we use Lemma~\ref{DG} (whose proof directly follows from the definition of $G(\FF_\vee)$).}{We first show how to decide whether $F\in \FF_\vee$ is a representative, i.e., whether $F\in G(\FF_\vee)$. The following follows immediately from the definition of $G(\FF_\vee)$.}
\begin{lemma}~\label{DG} Let $F\in \FF_\vee$. $F\in G(\FF_\vee)$ if and only if for every $f\in \FF\backslash S(F)$ we have
$F\vee f\not=F$.
\end{lemma}

\DanaFD{Lemma~\ref{imm} shows}{Now we show} how to decide whether $G'$ is an immediate descendant of $G$.
\begin{lemma}\label{imm} Let $G,G'\in G(\FF_\vee)$. $G'$ is an immediate descendant of $G$ if and only if $G'\not=G$, $S(G')\subset S(G)$ and for every $f\in S(G)\backslash S(G')$ we have $G'\vee f=G$.

If $G'\not=G$, $S(G')\subset S(G)$ and for some $f\in S(G)\backslash S(G')$ we have $G'\vee f\not= G$, then $G_{G'\vee f}$ is a descendant of $G$ and an ascendant of $G'$.
\end{lemma}
\begin{proof} \emph{Only if:} Let $G'$ be an immediate descendant of $G$, i.e., $G'\not=G$, $G'\imply G$ and $S(G')\subset S(G)$. Let $f\in S(G)\backslash S(G')$. Since $G'\imply (G'\vee f)\imply G$ and $G'\not=G'\vee f$, we get $G'\vee f=G$. 

\emph{If:} Suppose $G'\not=G$, $G'\imply G$ and for every $f\in S(G)\backslash S(G')$, we have $G'\vee f=G$. If $G'$ is not an immediate descendant of $G$, then let $G''$ be a descendant of $G$ and an immediate ascendant of $G''$. Take any $f\in S(G'')\backslash S(G')\subset S(G)\backslash S(G')$. Then, $G'\vee f=G''\not= G$ -- a contradiction. 
This also proves the last statement of Lemma~\ref{imm}.\qed
\end{proof}

Lemma~\ref{imm} shows how to compute an immediate descendant from a descendant, which we phrase in an operation called GetImmDe (Figure~\ref{getimmalgs}, left). GetImmDe takes $G$ and a descendant $G''$ of $G$ (which can even be the zero function), initializes $S=S(G'')$, and as long as possible, repeatedly extends $S$ as follows: for $f\in S(G)\backslash S$ if $(\vee S)\vee f\not=G$, $f$ is added to $S$.

\begin{figure}[t]
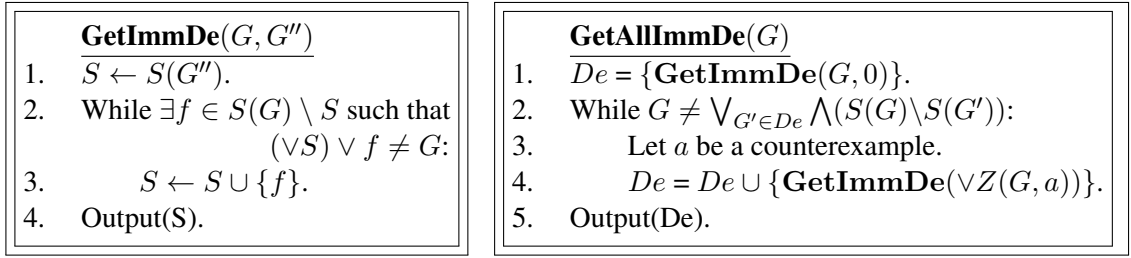

  {
   \begin{minipage}{.42\textwidth}
   \hspace{0.3cm}\fbox{\fbox{\begin{minipage}{28em}
  \begin{tabbing}
  xxxx\=xxxx\=xxxx\=xxxx\= \kill
  \> \underline{{\bf GetImmDe}$(G, G'')$}\\
  1. \>$S\gets S(G'')$.\\
  2. \> While $\exists f \in S(G)\setminus S$ such that \\\>\>\>\>\phantom{h}$(\vee S) \vee f \neq G$:\\
  3. \> \> $S \gets S \cup \{f\}$.\\
  4. \> Output(S).
  \end{tabbing}\end{minipage}}}
  \end{minipage}%
  \hspace{0.3cm}
    \begin{minipage}{0.45\textwidth}
   \fbox{\fbox{\begin{minipage}{28em}
  \begin{tabbing}
  xxxx\=xxxx\=xxxx\=xxxx\= \kill
  \> \underline{{\bf GetAllImmDe}$(G)$}\\
  1. \>$De$ = $\{{\bf GetImmDe}(G,0)\}$.\\
  2. \> While $G\neq\bigvee_{G' \in De}\bigwedge (S(G)\backslash S(G'))$:\\
  3. \> \> Let $a$ be a counterexample. \\
  4. \> \> $De$ = $De \cup \{{\bf GetImmDe}(\vee Z(G,a))\}$.\\
  5. \> Output(De).
  \end{tabbing}\end{minipage}}
  }
    \end{minipage}}
  \caption{Left: the GetImmDe operation. Right: the GetAllImmDe operation.}\label{getimmalgs}
  \end{figure}

\DanaFD{}{We now show how to \emph{compute} the descendants of $G$. First, we show how to compute one immediate descendant of $G$. This is done as follows.
Take any $S\subset S(G)$ such that $\vee S\not=G$ (even $S=\O$), and repeatedly extend $S$ as follows until no $f$ can be added: for $f\in S(G)\backslash S$ if $(\vee S)\vee f\not=G$, add $f$ to $S$.}

\DanaFD{GetImmDe can be used to obtain the first immediate descendant by calling it with $G''=0$. We next show how to obtain a descendant for which GetImmDe will return a different immediate descendant.}
{Next, we show how to compute a descendant.} To this end, we define the following:
For $G\in G(\FF_\vee)$ and a set $X'\subseteq X$, $Z(G,X')=\{f\in S(G)\ |\ f(X')=0\}$, where $f(X')=\vee_{x \in X'} f(x)$. When $X'=\{x\}$, we abbreviate to $Z(G,x)$. Obviously,
\begin{eqnarray}\label{int}
Z(G,X')=\bigcap_{x\in X'} Z(G,x).
\end{eqnarray}

Lemma~\ref{lem:zlem} relates this new definition to the descendants of $G$.
\begin{lemma}\label{lem:zlem}
Let $G\in G(\FF_\vee)$ and $X'\subseteq G^{-1}(1)$ be a nonempty set. Then, $G'=\vee Z(G,X')\in G(\FF_\vee)$ and $G'$ is a descendant of $G$.

For every immediate descendant $G'$ of $G$, there is $X'\subseteq G^{-1}(1)$ such that $\vee Z(G,X')=G'$.
\end{lemma}
\begin{proof} First notice that $G'(X')=0$.
Suppose on the contrary that $G'\not\in G(\FF_\vee)$. Then, there is $f\in S(G)\backslash Z(G,X')$ such that $G'\vee f=G'$.
Since $f\not\in Z(G,X')$, there is $z\in X'$ such that $f(z)=1$ and then $G'(X')=(G'\vee f)(X')\not=0$ -- a contradiction. Therefore, $G'\in G(\FF_\vee)$. By the definition of $Z$, $S(G') \subseteq S(G)$ and thus $G'$ is a descendant of $G$.

Let $G'$ be an immediate descendant of $G$ and let $X'= \{x \in X \mid G'(x)=0 \text{ and } G(x)=1\}$. Then, $X'\subseteq G^{-1}(1)$.
\DanaFD{We now show $Z(G,X')=S(G')$.
$S(G')\subseteq Z(G,X')$ because if $f \in S(G')$, then for all $x \in X'$, $f(x)=0$ and thus $f\in Z(G,X')$. We next prove that $Z(G,X')\subseteq S(G')$. Let $f\in Z(G,X')$ and $x_0$ be a witness for $G$ and $G'$, i.e., $G(x_0)=1$ and $G'(x_0)=0$. Therefore, $x_0\in X'$ and $f(x_0)=0$ by the definition of $Z$. By Lemma~\ref{wit1}, for every $f\in S(G)\backslash S(G')$, we have $f(a)=1$. Since $f\in S(G)$, it must be that $f\in S(G')$.}{and $Z(G,X')$ contains all the functions in $S(G')$ and no function in $S(G)\backslash S(G')$. Therefore $\vee Z(G,X')=G'$.}\qed
\end{proof}

\DanaFD{Lemma~\ref{lem:zlem} shows how to construct descendants from elements in $X$. Lemma~\ref{allim} determines when all immediate descendants of $G$ were obtained, and if not, how to obtain a new element in $X$ that leads to a new immediate descendant. The algorithm that finds all immediate descendants (Figure~\ref{getimmalgs}, right) follows directly from this lemma. In the following we denote $\bigwedge S=\bigwedge_{f\in S} f$.}
{We now prove}
\begin{lemma}\label{allim} Let $G_1,\ldots,G_m$ be immediate descendants of $G$. There is no other immediate descendant for $G$ if and only if
\vspace{-0.5cm}
\begin{eqnarray}\label{constar}
G=\bigvee_{i=1}^m\bigwedge (S(G)\backslash S(G_i)).
\end{eqnarray}
If (\ref{constar}) does not hold, then for any counterexample $a$ for (\ref{constar}), we have $\vee Z(G,a)$ is a descendant of $G$ but not equal to and not a descendant of any $G_i$, $i=1,\ldots,m$.
\end{lemma}
\begin{proof} \emph{Only if:} Suppose $G\not=\vee_{i=1}^m\wedge (S(G)\backslash S(G_i))$ and let $a$ be a counterexample. Since for all $i$, $S(G)\backslash S(G_i)\subseteq S(G)$, we have $\vee_{i=1}^m\wedge (S(G)\backslash S(G_i))\imply G$. Therefore, $G(a)=1$ and for every $i$ there is $f_i\in S(G)\backslash S(G_i)$ such that $f_i(a)=0$.
Consider $G'=Z(G,a)$. Since $f_i(a)=0$, we have $f_i\in S(G')$. Since $f_i\not\in S(G_i)$, $G'$ is not a descendant of $G_i$. Since $G'$ is a descendant of $G$ and not a descendant of any $G_i$, there must be another immediate descendant of $G$.

\emph{If:} Denote $W=\bigvee_{i=1}^m\bigwedge (S(G)\backslash S(G_i))$. Let $G'$ be another immediate descendant of $G$. Let $a$ be a witness for $G$ and $G'$. Then $G(a)=1$ and, by Lemma~\ref{wit1}, for every $f\in S(G)\backslash S(G')$, we have $f(a)=1$ and for every $f\in S(G')$, we have $f(a)=0$. Since $S(G')\not\subset S(G_i)$, we have $S(G)\backslash S(G_i)\not\subset S(G)\backslash S(G')$, and therefore, $(S(G)\backslash S(G_i))\cap S(G')$ is not empty. Choose $f_i\in (S(G)\backslash S(G_i))\cap S(G')$. Then, $f_i\in S(G)\backslash S(G_i)$ and since $f_i\in S(G')$, $f_i(a)=0$. Therefore, $W(a)=0$. Since $G(a)=1$, we get $G\not= W$. 
\end{proof}

\DanaFD{}{Now let $G_1,\ldots,G_m$ be immediate descendants of $G$. If (\ref{constar}) holds then there is no other immediate descendant. If not then let $a$ be a counterexample. Find $\vee Z(G,a)$. Then use the previous procedure to climb up the Hasse diagram until you get a new immediate descendant. This happens because $\vee Z(G,a)$ is not a descendant of any of the $G_i$, $i=1,\ldots,m$.}

\subsection{Critical Points}
In this section, we show that if one can find certain points (the \emph{critical points}), then the immediate descendants can be computed in polynomial time (in the number of these points) and thus SPEX runs in polynomial time. In particular, if the number of points is polynomial and they can be obtained in polynomial time, then SPEX runs in polynomial time. 

A set of points $C\subseteq X$ is called {\it a critical point set} for $\FF$ if for every $S\subseteq \FF$ and if $H=\bigwedge_{f\in S}f\wedge \bigwedge_{f\in \FF\backslash S}\bar{f}\not=0$, then there is a point in $c\in C$ such that $H(c)=1$.

We now show how to use the critical points to find the immediate descendants.

\begin{lemma}\label{lemcriticaldes}
\DanaD{Let $C\subseteq X$ be a set of points.} If $C$ is a set of critical points for $\FF$, then all the immediate descendants of $G\in G(\FF_\vee)$ can be found in time $|C|\cdot |S(G)|$.
\end{lemma}
\begin{proof}
Let $G_1,\ldots,G_m$ be some of the immediate descendants of $G$. To find another immediate descendant we look for a point $a$ such that $G(a)=1$, and for every $i$, there is $f_i\in S(G)\backslash S(G_i)$ such that $f_i(a)=0$. Let $a$ be such point. Consider $S=\{f\in \FF\ |\ f(a)=1\}$ and let $H= \bigwedge_{f\in S}f\wedge \bigwedge_{f\in \FF\backslash S}\bar{f}$. Since $H(a)\not=0$, there is a critical point $b\in C$ such that $H(b)=1$. By the definition of $b$, $G(b)=1$ and for every $i$ there is $f_i\in S(G)\backslash S(G_i)$ such that $f_i(b)=0$. Therefore, $b$ can be used to find a new descendant of $G$.
\DanaD{To find the descendants we need, in the worst case, to substitute all the assignments of $C$ in all the descendants $G_1,\ldots,G_m$ and $G$. For all the descendants, this takes at most $|C|\cdot |S(G)|$ steps, which implies the time complexity.}
\qed 
\end{proof}

We now show how to generate the set of critical points.
\begin{lemma}\label{Cri} If for every $S,R\subseteq \FF$ one can decide whether $H_{S,R}=\bigwedge_{f\in S}f\wedge \bigwedge_{f\in R}\bar{f}\not=0$ in time $T$ and if so, find $a\in X$ such that $H_{R,S}(a)=1$, then a set of critical points $C$ can be found in time $|C|\cdot T\cdot |\FF|$.
\end{lemma}
\begin{proof} The set is constructed inductively, in stages. Let $\FF=\{f_1,\ldots,f_t\}$ and denote $\FF_i=\{f_1,\ldots,f_i\}$. Suppose we have a set $K_i=\{S\subseteq [i] \mid \bigwedge_{f\in S}f\wedge \bigwedge_{f\in [i]\backslash S}\bar{f}\not=0\}$; then we define $K_{i+1}=\{ g\wedge f_{i+1}\ |\  g\in K_i \mbox{\ and\ } g\wedge f_{i+1}\not=0\}$ $\cup \{ g\wedge \overline{f_{i+1}}\ |\  g\in K_i \mbox{\ and\ } g\wedge \overline{f_{i+1}}\not=0\}$.
\end{proof}

\section{A Polynomial Time Algorithm for Halfspaces in a Constant Dimension}\label{sec:halfs}
In this section, we show two results. The first is that when $\FF$ is a set of halfspaces over a constant dimension, one can find a polynomial-sized critical point set in polynomial time, and therefore, SPEX can run in polynomial time. We then show that unless $P=NP$, this result cannot be extended to non-constant dimensions.

A halfspace of dimension $d$ is a boolean function of the form: $$f(x_1,\ldots,x_d)=[a_1x_1+\cdots+a_dx_d\ge b]=\left\{ \begin{array}{ll}
1&\mbox{\ \ if \ }a_1x_1+\cdots+a_dx_d\ge b \\ 0&\mbox{\ \ otherwise}\end{array}\right.$$ where $(x_1,\ldots,x_d)\in \Re^d$ and $a_1,\ldots,a_d,b$ are real numbers. Therefore, $f:\Re^d\to \{0,1\}$.

We now prove that the set of critical points is of a polynomial size.
\begin{lemma} Let $\FF$ be a set of halfspaces in dimension $d$. There is a set of critical points $C$ for $\FF$ of size $|\FF|^{d+1}.$
\end{lemma}
\begin{proof} Define the dual set of halfspaces. That is, for every $x\in \Re^d$, the dual function $x^\bot:\FF\to \{0,1\}$ where $x^\bot(f)=f(x)$. It is well known that the VC-dimension of this set is at most $d+1$. By the Sauer-Shelah lemma, the result follows.\qed
\end{proof}

Next, we prove that the set of critical points can be computed in polynomial time.
\begin{lemma} Let $\FF$ be a set of halfspaces in dimension $d$. A set of critical points $C$ for $\FF$ of size $|\FF|^{d+1}$ can be found in time $poly(|\FF|^d)$.
\end{lemma}
\begin{proof} Follows from Lemma~\ref{Cri} and the fact that linear programming (required to check whether $g\wedge f_i \neq 0$) takes polynomial time.\qed
\end{proof}

By the above results, we conclude:
\begin{theorem} Let $\FF$ be a set of halfspaces in dimension $d$. There is a learning algorithm for $\FF_\vee$ that runs in time $|\FF|^{O(d)}$ and asks at most $|\FF|\cdot \OPT(\FF_\vee)$ membership queries.

In particular, when the dimension $d$ is constant, the algorithm runs in polynomial time.
\end{theorem}

Next, we show that the above cannot be extended to a non-constant dimension:
\begin{theorem} If every set $\FF$ of halfspaces deciding whether $F\in \FF_\vee$ is a descendant of $\vee\FF$ can be done in polynomial time, then $P=NP$.
\end{theorem}
\begin{proof}
 The reduction is from the problem of dual 3SAT -- that is, given the literals $\{x_1,\ldots,x_n,\bar{x}_1,\ldots,\bar{x}_n\}$ and the terms $T_1,\ldots,T_m$ where each $T_i$ is a conjunction of three literals, decide whether $T_1\vee\cdots\vee T_m= 1$.

Given the terms $T_1,\ldots,T_m$, each can be translated into a halfspace. For example, the term $x_1\wedge \bar{x}_2\wedge x_3$ corresponds to the halfspace $[x_1+(1-x_2)+x_3\ge 3]=[x_1-x_2+x_3\ge 2]$. Now, consider $\FF=\{T_1,\ldots,T_m,1\}$. Then $G_{\max}=1$ and $T_1\vee\ldots\vee T_m\not=1$ if and only if $T_1\vee \cdots\vee T_m$ is the only immediate descendant of $G_{\max}$.\qed

\end{proof}

\section{Duality and a Polynomial Time Algorithm for Variable Inequality Predicates}\label{sec:vars}

In this section, we study the learnability of conjunctions over variable inequality predicates.
In the acyclic case, we provide a polynomial time learning algorithm. In the general case, we show that the learning problem is equivalent to the open problem of enumerating all the maximal acyclic subgraphs of a given directed graph.

Consider the set of boolean functions $\FF^I:=\{ [x_i>x_j]\ | (i,j)\in I\}$ for some $I\subseteq [n]^2$ where $[n]=\{1,2,\ldots,n\}$ and the variables $x_i$ are interpreted as real numbers. We define $[x_i>x_j]=1$ if $x_i>x_j$; and $0$ otherwise. We assume throughout this section that $(i,i)\not\in I$ for all $i$.

We consider the dual class $\FF_\wedge:=\{\wedge_{f\in S} f\  |\ S\subseteq \FF\}$.
By duality (De Morgan's law), all our results are true for learning $\FF_\wedge$ (after swapping $\vee$ with $\wedge$). The dual SPEX algorithm is depicted in Figure~\ref{Alg1}.

For a set $J\subseteq I$, we define $F_J=\wedge_{(i,j)\in J}[x_i>x_j]$. For $F\in \FF^I_\wedge$ we define $\I(F)=\{(i,j)\mid [x_i>x_j]$ is in $F\}$. Note that $\I(F_J)=J$. For example, $\I([x_1>x_2]\wedge [x_3>x_1])=\{(1,2),(3,1)\}$.

The {\it directed graph of} $I\subseteq [n]^2$ is ${\cal G}_I=([n],I)$.  The {\it reachability matrix} of $I$, denoted by $R(I)$, is an $n\times n$ matrix where $R(I)_{i,j}=1$ if there is a (directed) path from $i$ to $j$ in ${\cal G}_I$; otherwise, $R(I)_{i,j}=0$.
We say that $I$ is {\it acyclic} (resp., {\it cyclic}) if the graph ${\cal G}_I$ is acyclic (resp., cyclic). 
We say that an assignment to the variables $a\in [n]^n$ is a {\it topological sorting} of $I$ if for every $(i,j)\in I$, we have $a_i>a_j$. It is known that $I$ has a topological sorting if and only if $I$ is acyclic. Also, it is known that a topological sorting for an acyclic set can be found in linear time (see \cite{K}, Volume 1, Section 2.2.3 and \cite{CLRS}). Next, we study the learnability of $\FF^I_\wedge$ when $I$ is acyclic and following this, we discuss the general case.

\subsection{Acyclic Sets}
We now examine the case when $I$ is acyclic. Here, the number of critical points of $\FF^I$ where $I=\{(1,2),(2,3),$ $\cdots,(n-1,n)\}$ is $2^{n-1}$. Therefore, using Lemma~\ref{lemcriticaldes} does not enable us to obtain a polynomial time algorithm. 
Accordingly, we show a different way to determine whether a function is a representative (i.e., in $G(\FF_\wedge)$), and then show how to obtain the immediate descendants in quadratic time (in $n$) and the witnesses in linear time. As a result, SPEX can run in polynomial time. Finally, we show that the number of membership queries is at most $|I|$.

Before we show the main lemma, Lemma~\ref{D02}, we present Lemma~\ref{sec6:firstlem}, a trivial lemma that we use to prove Lemma~\ref{D02}.
\begin{lemma} \label{sec6:firstlem}
Let $I\subseteq [n]^2$ be an acyclic set, $F\in \FF^I_\wedge $, and $a\in [n]^n$. Then, $F(a)=1$ if and only if $a$ is a topological sorting of $\G_{\I(F)}$.

In particular, $F$ is satisfiable and a satisfying assignment $a\in [n]^n$ can be found in linear time.
\end{lemma}
\DanaFD{}{\begin{proof} Trivial.\qed
\end{proof}}

Next we show our main lemma that enables us to determine the representative elements and the immediate descendants (in Lemmas~\ref{lem:mainlem}--\ref{fWG}). The proof is provided in Appendix~\ref{sec6proofs}.

\begin{lemma}\label{D02} Let $I$ be acyclic and $F_1,F_2\in \FF^I_\wedge$. Then, $F_1=F_2$ if and only if $R(\I(F_2))$$=$$R(\I(F_1))$.
\end{lemma}

We now show how to decide whether $F\in G(\FF^I_\wedge)$ -- that is, whether $F$ is a representative.
\begin{lemma}\label{lem:mainlem} Let $I$ be an acyclic set and $F\in \FF^I_\wedge$. $F\in G(\FF^I_\wedge)$ if and only if for every $(i,j)\in I\backslash \I(F)$ there is no path from $i$ to $j$ in $\G_{\I(F)}$.
\end{lemma}
\begin{proof} \emph{If:} If for every $(i,j)\in I\backslash \I(F)$ there is no path from $i$ to $j$ in $\G_{\I(F)}$, then for every $(i,j)\in I\backslash \I(F)$, $R(\I(F)\cup\{(i,j)\})\not=R(\I(F))$ \DanaFD{. By}{which by} Lemma~\ref{D02}, this implies that $F\wedge [x_i>x_j]\not= F$. By (the dual result of) Lemma~\ref{DG}, the result follows.

\emph{Only if:} Now let $F\in G(\FF^I_\wedge)$. By Lemma~\ref{DG}, for every $[x_i>x_j]\not\in F$, we have $F\wedge [x_i>x_j]\not= F$. Therefore, there is an assignment $a$ that satisfies $a_i\le a_j$ and $F(a)=1$. As before, if there is a path in $\G_{\I(F)}$ from $i$ to $j$, then we get a contradiction.
\qed
\end{proof}

We now show how to determine the immediate descendants of $G$ in polynomial time.
\begin{lemma}\label{DI}
Let $I$ be acyclic. The immediate descendants of $G\in G(\FF_\wedge^I)$ are all $G^{r,s}:=F_{ \I(G)\backslash \{(r,s)\}} $ where $(r,s)\in \I(G)$ and there is no path from $r$ to $s$ in $\G_{\I(G)\backslash \{(r,s)\}}$.

In particular, for all $G\in G(\FF_\wedge^I)$, we have $|\De(G)|\le |\I(G)|\le |I|$.
\end{lemma}

The proof is in Appendix~\ref{sec6proofs}. We now show how to find a witness.
\begin{lemma}\label{fWG}
Let $I$ be acyclic, $G\in G(\FF_\wedge^I)$, and $G^{r,s}:=F_{ \I(G)\backslash \{(r,s)\}} $ be an immediate descendant of $G$.
A witness for $G$ and $G^{r,s}$ can be found in linear time.
\end{lemma}
\begin{proof} By Lemma~\ref{DI}, $(r,s)\in\I(G)$ and there is no path from $r$ to $s$ in $\G_{\I(G)\backslash \{(r,s)\}}$. Therefore, if we match vertices $r$ and $s$ in $\G_{\I(G)\backslash \{(r,s)\}}$ we get an acyclic graph $\G'$. Then, a topological sorting $a$ for $\G'$ is a satisfying assignment for $G^{r,s}$ that satisfies $a_r=a_s$. Since $[x_r>x_s]\in S(G)$, we get $G(a)=0$. Therefore, $a$ is a witness for $G$ and $G^{r,s}$.\qed
\end{proof}

To learn a function in $\FF^I_\wedge$, SPEX needs to find the immediate descendants of $G$ and a witness for each immediate descendant and $G$. By Lemma~\ref{DI}, this involves finding a path between every two nodes in the directed graph $\G_{I(G)}$, which can be done in polynomial time. By Lemma~\ref{fWG}, to find a witness, SPEX needs a topological sorting, which can be done in linear time. Therefore, SPEX runs in polynomial time.
Therefore, by Theorem~\ref{upper} and Lemma~\ref{DI}, the class $\FF^I_\wedge$ is learnable in polynomial time with at most $|I|^2$ membership queries. We now show that the number of membership queries is actually lower and equal to $|I|$.

\begin{theorem}\label{theorem:acyclic} Let $I\subseteq [n]^2$ be acyclic. The class $\FF^I_\wedge$ is learnable in polynomial time with at most $|I|$ membership queries.
\end{theorem}
\begin{proof}
Consider the (dual) Algorithm SPEX in Figure~\ref{Alg1} in Appendix~\ref{app:dual}. Let $F$ be the target function. Let $G_{\max} = f_1\wedge f_2\wedge\cdots \wedge f_t$ where $f_i\in \FF^I$. By Lemma~\ref{DI}, we may assume w.l.o.g. that $G^{(i)}=f_1\wedge f_2\wedge\cdots \wedge f_{i-1}\wedge f_{i+1}\wedge\cdots f_t$ where $i=1,\ldots,\ell$ are all the immediate descendants of $G$. Let $a^{(i)}$ be the witness for $G$ and $G^{(i)}$, $i=1,\ldots,\ell$.

In the algorithm, $S=\{f_i\ |\ i=1,\ldots,t\}$. If $F(a^{(i)})=1$, then Line 5 in the algorithm removes $f_i$ from $S$ and \DanaFD{$f_i$}{it will} never returns \DanaFD{}{back} to $S$. If $F(a^{(i)})=0$, then the set $\{f_1,f_2,\ldots, f_{i-1}, f_{i+1},\cdots f_t\}$ is added to $T$, which means (see Line~3) that SPEX never considers a descendant that does not contain $f_i$. Namely, for every $f_i$, SPEX makes at most one membership query.\qed
\end{proof}

\DanaD{We conclude this section by illustrating SPEX on an example, depicted in Figure~\ref{Graph}. Assume the set of boolean functions is $\FF^I$, where $I=\{(1,2),(1,4),(1,3),(3,4),(2,4),(3,2)\}$, and the target is $G_{min}=1$.
 The graph in Figure~\ref{Graph} shows the Hasse diagram (in white and gray nodes) and the candidates that SPEX considers (in gray).
 The figure demonstrates that the number of membership queries is equal to $|I|$.
 }

\subsection{Cyclic Sets}
In this section, we consider the general case, where $I\subseteq [n]^2$ can be any set.
Lemma~\ref{IDG} shows a few results when $I$ is cyclic.
\begin{lemma}\label{IDG}  Let $I\subseteq [n]^2$ be any set with cycles. Then:
\begin{enumerate}[nosep,nolistsep]
\item $G_{\max}=0$ is in $\FF^I_\wedge$.
\item The immediate descendants of $G_{\max}$ are all $\wedge_{(i,j)\in J}[x_i>x_j]$  where $\G_{J}$ is a maximal acyclic subgraph of $\G_{I}$.

In particular,
\item Finding all the immediate descendants of $G_{\max}$ is equivalent to enumerating all the maximal acyclic subgraphs of $\G_{I}$.
\end{enumerate}
\end{lemma}
\begin{proof} If $i_1\to i_2\to\cdots\to i_c\to i_1$ is a cycle, then $G_{\max}\imply [x_{i_1}>x_{i_1}]=0$ and thus $G_{\max}=0$.

If $\G_{J}$ is a maximal acyclic subgraph of $\G_{I}$, then adding any edge in $I\backslash J$ to $\G_{J}$ \DanaFD{creates}{produce} a cycle. This implies that for any $[x_i>x_j]\in S(F_I)\backslash S(F_J)$, we have $F_J\wedge [x_i>x_j]=0=G_{\max}$. By Lemma~\ref{imm}, $F_J$ is an immediate descendant of $G_{\max}$.

Now, if $F_J$ is an immediate descendant of $G_{\max}$, then $J$ is acyclic because otherwise $F_J=0=G_{\max}$.
If $\G_J$ is not a maximal acyclic subgraph of $\G_I$, then there is an edge $(i,j)$ such that $J\cup \{(i,j)\}$ is acyclic and then either $F_{J\cup \{(i,j)\}}=F_J$ -- in which case $F_{J}$ is not a representative and thus not an immediate descendant -- or $F_{J\cup \{(i,j)\}}\not =F_J$ -- in which case $G_{\max}\imply F_{J\cup \{(i,j)\}}\imply F_J$ and $G_{\max}\not= F_{J\cup \{(i,j)\}}\not= F_J$, and therefore, $F_J$ is not an immediate descendant of $G_{\max}$.\qed
\end{proof}

Let $\G$ be any directed graph and denote by $N(\G)$ the number of the maximal acyclic subgraphs of $\G$. Lemma~\ref{loweri} follows immediately from Theorem~\ref{lower} and Lemma~\ref{IDG}.

\begin{lemma}\label{loweri} $\OPT(\FF^I_\wedge)\ge N(\G_I).$
\end{lemma}

The problem of enumerating all the maximal acyclic subgraphs of a directed graph is still an open problem \citep{ABC12,BCL13,W16}. We show that learning a function in $\FF_\wedge^I$ (where $I\subseteq [n]^2$) in polynomial time is possible if and only if the enumeration problem can be done in polynomial time (the proof is in Appendix~\ref{sec6proofs}).

\begin{theorem}\label{lasttheo} There is a polynomial time learning algorithm (poly$(\OPT(\FF^I_\wedge),n,$ $|I|)$), which for an input $I\subseteq [n]^2$, learns $F\in \FF^I_\wedge$ if and only if there is an algorithm that for an input $\G$, which is a directed graph, enumerates all the maximal acyclic subgraphs of $\G(V,E)$ in polynomial time (poly$(N(\G),|V|,|E|)$).
\end{theorem}

\section{Application to Program Synthesis}\label{sec:synthesis}
In this section, we explain the natural integration of SPEX into program synthesis. We then demonstrate this on a synthesizer that synthesizes programs that detect patterns in time-series charts. These programs meet specifications that belong to the class of variable inequalities $I$ (for acyclic~$I$).

Program synthesizers are defined over an input domain $X_{in}$, an output domain $X_{out}$, and a domain-specific language $D$.
Given a specification, the goal of a synthesizer is to generate a corresponding program.
A specification is a set of formulas $\varphi(x_{in},x_{out})$ where $x_{in}$ is interpreted over $X_{in}$ and $x_{out}$ is interpreted over $X_{out}$.
 Given a specification $Y$, a synthesizer returns a program $P:X_{in}\rightarrow X_{out}$ over $D$ such that for all $in\in X_{in}$: $(in,P(in)) \models Y$ (i.e., all formulas are satisfied for $x_{in}=in$ and $x_{out}=out$).
 Roughly speaking, there are two types of synthesizers:
 \begin{itemize}[nosep,nolistsep]
   \item Synthesizers that assume that $Y$ describes a full specification. Namely, for all $in \in X_{in}$, there exists a single $out\in X_{out}$ such that $(in,out)\models Y$ (e.g., ~\cite{Sketch:PLDI08,BoardSingh:2011:SDS,alur-fmcad13,BornholtTGC16}).
   \item Synthesizers that assume that $Y$ describes only input--output examples (known as PBE synthesizers). Namely, all formulas in the specification take the form of $x_{in}=in\imply x_{out}=out$ (e.g.,~\cite{Gulwani:2011,Polozov:2015,Barowy:2015}). The typical setting of a PBE synthesizer is that an end user (that acts as the teacher) knows a target program $f$ and he or she provides the synthesizer with some initial examples and can interact with the synthesizer through membership queries (we note that most synthesizers do not interact).
 \end{itemize}

 Each approach has its advantages and disadvantages. The first approach guarantees correctness on all inputs, but requires a full specification, which is complex to provide, especially by end users unfamiliar with formulas. On the other hand, PBE synthesizers are user-friendly as they interact through examples; however, generally they do not guarantee correctness on all inputs.

   We next define the class of programs that are \emph{$\FF$-describable}. For such programs, SPEX can be leveraged by both approaches to eliminate their disadvantage.
Let $\FF$ be a set of predicates over $X_{in}\cup X_{out}$. A synthesizer is $\FF$-describable if every program that can be synthesized meets a specification $F\in \FF_{\vee}$ (or dually, $\FF_{\wedge}$).
A synthesizer that assumes that $Y$ is a full specification and is $\FF$-describable can release the user from having to provide the full specification by first running SPEX and then synthesizing a program from the formula returned by SPEX. A PBE synthesizer that is $\FF$-describable can be extended to guarantee correctness on all inputs by first running SPEX and then synthesizing the program from the set of membership queries posed by SPEX.
Theorem~\ref{the:synthesis} follows immediately from Theorem~\ref{upper}.

\begin{theorem}\label{the:synthesis} Let $\FF$ be a set of predicates and $\mathcal{A}$ be an $\FF$-describable synthesizer. Then, $\mathcal{A}$ extended with SPEX returns the target program with at most $|\FF|\cdot \max_{G\in G(\FF_\vee)} |\De(G)|$ membership queries.
\end{theorem}

\subsection{Example: Synthesis of Time-series Patterns}
In this section, we consider the setting of synthesizing programs that detect time-series patterns.
The specifications of these programs are over $\FF_\wedge^I$ for some $I\subseteq [n]^2$ and thus the synthesizer learns the target program within $|I|$  membership queries. Time-series are used in many domains including financial analysis \citep{EncPatterns}, medicine \citep{Chuah2007}, and seismology \citep{MoralesEsteban}. Experts use these charts to predict important events (e.g., trend changes in stock prices) by looking for \emph{patterns} in the charts. There are a variety of platforms that enable users to write programs to detect a pattern in a time-series chart. In this work, we consider a domain-specific language (DSL) of a popular trading platform,~\cite{AmiBroker}. Our synthesizer can easily be extended to other DSLs.

A time-series chart $c:\N \rightarrow \R$ maps points in time to real values (e.g., stock prices). A time-series pattern is a conjunction $F_I$ for acyclic $I$. The size of $F_I$ is the maximal natural number it contains, i.e., the size of $F_I$ is $argmax_i\{i \mid \exists j.(i,j)\in I \text{ or } (j,i)\in I\}$. A program detects a pattern $F_I$ of size $k$ in a time-series chart $c$ if it alerts upon every $t \in \N$ for which the $t_1,...,t_{k-1}$ preceding extreme points satisfy $F_I(c(t_1),...,c(t_{k-1}),c(t))=1$.

In this setting, $X_{in}$ is a set of charts over a fixed $k\in \N$, that is $f:\{1,...,k\}\rightarrow \R$, and $X_{out} =\{0,1\}$. The DSL $D$ is the DSL of the trading platform~\cite{AmiBroker}.
We built a synthesizer that not only interacts with the end user through membership queries, but also displays them as charts (of size $k$). Thereby, our synthesizer communicates with the end user in his language of expertise. The synthesizer takes as input an initial chart example $c':\{1,...,k\}\rightarrow \R$ and initializes $I$ to $\{(i,j)\in [k]^2\mid c'(i) \geq c'(j)\}$ and sets $\FF^I:=\{ [x_i\geq x_j]\ | (i,j)\in I\}$ (our results are true also for these kinds of predicates). It then executes SPEX to learn $F$. During the execution, every witness is translated into a chart (the translation is immediate since each witness is an assignment to $k$ points). Finally, our synthesizer synthesizes a program by synthesizing instructions that detect the $k$ extreme points in the chart, followed by an instruction that checks whether these points satisfy the formula $F$ and alerts the end user if so (the technical details are beyond the scope of this paper). Namely, for the end user, our synthesizer acts as a PBE synthesizer, but internally it takes the first synthesis approach and assumes it is given a full specification (which is obtained by running SPEX). The complexity of the overall synthesis algorithm is determined by SPEX (as the synthesis merely synthesizes the instructions according to the specification $F$), and thus from Theorem~\ref{theorem:acyclic}, we infer the following theorem.

\begin{theorem} The pattern synthesizer returns a program that detects the target pattern in polynomial time with at most $k^2$ membership queries, where $k$ is the target pattern size.
\end{theorem}

\section{Related Work}

\sloppy
\paragraph{Program Synthesis}
Program synthesis has drawn a lot of attention over the last decade, especially in the setting of synthesis from examples, known as PBE
(e.g.,~\cite{Gulwani:2010,Lau03,Sarma:2010,Harris:PLDI11,Gulwani:2011,Gulwani:CACM12,Singh:VLDB12,Yessenov13,Recursive:2013,Sai:2013,Aditya13,FlashExtract:14,Barowy:2015,Polozov:2015}). Commonly, PBE algorithms synthesize programs consistent with the examples, which may not capture the user intent.
Some works, however, guarantee to output the target program. For example, CEGIS~\citep{solar2008program} learns a program via equivalence queries, and oracle-based synthesis~\citep{BitManipulation:2010}
assumes that the input space is finite, which allows it to guarantee correctness by exploring all inputs (i.e., without validation queries).
Synthesis has also been studied in a setting where a specification and the program's syntax are given and the goal is to find a program over this syntax meeting the specification (e.g.,~\cite{Sketch:PLDI08,BoardSingh:2011:SDS,alur-fmcad13,BornholtTGC16}).

\paragraph{Queries over Streams} Several works aim to help analysts.
Many trading software platforms provide domain-specific
languages for writing queries where the user defines the query and the system is responsible for the sliding window mechanism,
e.g., \emph{\href{www.metaquotes.net}{MetaTrader}},
\emph{\href{www.metastock.com}{MetaStock}},
\emph{\href{www.ninjatrader.com}{NinjaTrader}}, and
\emph{Microsoft's StreamInsight}~\citep{AFA}.
Another tool designed to help analysts is \emph{Stat!}~\citep{Stat:2013}, an interactive tool enabling analysts to write queries in StreamInsight. \emph{TimeFork}~\citep{BadamZSEE16} is an interactive tool that helps analysts with predictions based on automatic analysis of the past stock price.
\emph{CPL}~\citep{CPL} is a Haskell-based high-level language designed for
chart pattern queries.
Many other languages support queries for streams.
\emph{SASE}~\citep{Wu:2006} is a system designed for RFID (radio frequency identification) streams that offers a user-friendly language and can handle large volumes of data. \emph{Cayuga}~\citep{Brenna:2007} is a system for detecting complex patterns in streams, whose language is based on Cayuga algebra. \emph{SPL}~\citep{Hirzel:2013} is IBM's stream processing language supporting pattern detections. \emph{ActiveSheets}~\citep{vaziri2014stream} is a platform that enables Microsoft Excel to process real-time streams from within spreadsheets.

\section{Conclusion}
In this paper, we have studied the learnability of disjunctions $\FF_\vee$ (and conjunctions) over a set of boolean functions $\FF$. We have shown an algorithm SPEX that asks at most $|\FF|\cdot OPT(\FF_\vee)$ membership queries. We further showed two classes that SPEX can learn in polynomial time. We then showed a practical application of SPEX that augments PBE synthesizers, giving them the ability to guarantee to output the target program as the end user intended. Lastly, we showed a synthesizer that learns time-series patterns in polynomial time and outputs an executable program, while interacting with the end user through visual charts.
\paragraph{Acknowledgements}
The research leading to these results has received funding from the European Union's - Seventh Framework Programme (FP7) under grant agreement {n$^o$}~615688--ERC-COG-PRIME.

\bibliography{bib}
\appendix
\section{The Dual SPEX Algorithm}\label{app:dual}
Figure~\ref{Alg1} shows the dual SPEX algorithm for learning functions in $\FF_\wedge$.
\begin{figure}[h!]
  \begin{center}
   \fbox{\fbox{\begin{minipage}{28em}
  \begin{tabbing}
  xxxx\=xxxx\=xxxx\=xxxx\= \kill
  \>\underline{{\bf Algorithm: Dual SPEX -- The target function is $F$.}}\\ \\
  \> Learn($G\gets G_{\max}$,$T\gets \O$).\\ \\
  \> \underline{{\bf Learn}$(G,T)$}\\
  1. \>$S\gets S(G)$; Flag$=1$.\\
  2. \> For every immediate descendant $G'$ of $G$:\\
  3. \> \> If $S(G')\not\subset R$ for all $R\in T$ then:\\
  4. \> \> \> Find a witness $a$ for $G'$ and $G$.\\
  5. \> \> \> If $F(a)={\color{blue} 1}$ then: $S\gets S\cap S(G')$; Flag$=0$.\\
  6. \> \> \> If $F(a)={\color{blue} 0}$ then: $T\gets T\cup \{S(G')\}$.\\
  7. \>\> EndIf\\
  8. \> EndFor\\
  9. \> If Flag$=1$ then: Output($\wedge S$)\\
  10. \>\> Else Learn$(\wedge S, T)$.\\
  \end{tabbing}\end{minipage}}}
  \end{center}
	\caption{The dual algorithm of SPEX for learning functions in $\FF_\wedge$.}
	\label{Alg1}
	\end{figure}
\section{Proofs for Section~\ref{sec:32}}\label{sec3proofs}

{\bf Proof of Lemma \ref{fact}}
Consider $G_F$. Since $F=G_F$, $G_1\imply G_F\imply G_2$. By the definition of immediate descendants, we get the result.\qed\\
{\bf Proof of Lemma \ref{trivial}}
Bullet {\it 1}: Consider $G_F$. Then $G_F=F$ and $G_F\in G(\FF_\vee)$. Since $G_1,G_2\imply G_F\imply $ $\lca(G_1,G_2)$, by the definition of LCA we must have $G_F=\lca(G_1,G_2)$.
The proof of {\it 2} is similar.\qed\\
{\bf Proof of Lemma \ref{lca}}
\sloppy
Since $G_1,G_2\imply \lca(G_1,G_2)$, we get $G_1\vee G_2\imply \lca(G_1,G_2)$.
Since $G_1,G_2\imply (G_1\vee G_2)\imply \lca(G_1,G_2)$, by Lemma~\ref{trivial}, we get
$G_1\vee G_2=\lca(G_1,G_2)$.\qed\\
{\bf Proof of Lemma \ref{gcd}}
Let $G=\gcd(G_1,G_2)$.
We show that $S(G)\subseteq S(G_1)\cap S(G_2)$ and $S(G_1)\cap S(G_2) \subseteq S(G)$.
By Lemma~\ref{GG}, $S(G)\subseteq S(G_1)$ and $S(G)\subseteq S(G_2)$, and therefore, $S(G)\subseteq S(G_1)\cap S(G_2)$.
Since $S(G)\subseteq S(G_1)\cap S(G_2)$, we also have $G=\vee S(G)\imply \vee(S(G_1)\cap S(G_2)) \imply $ $G_1,G_2$. Therefore, by Lemma~\ref{trivial}, we get $G = \vee (S(G_1)\cap S(G_2))$. Thus, $S(G_1)\cap S(G_2) \subseteq S(G)$.\qed

\section{Additional Proofs for Section~\ref{sec:vars}}\label{sec6proofs}
{\bf Proof of Lemma \ref{D02}} \emph{Only if:} Assume $F_1=F_2$. Suppose, on the contrary, that there are $i$, $j$ such that w.l.o.g. $R(\I(F_1))_{i,j}=0$ and $R(\I(F_2))_{i,j}$ $=1$. Since $I$ is acyclic and $R(\I(F_2))_{i,j}=1$, there is no path from $j$ to $i$ in ${\cal G}_I$ (and therefore, in $\G_{\I(F_1)}$). Since $R(\I(F_1))_{i,j}=0$, there is also no path from $i$ to $j$ in $\G_{\I(F_1)}$.
Therefore, \DanaFD{we can}{if we} match \DanaFD{the vertices}{vertex} $i$ and $j$ in $\G_{\I(F_1)}$ (\DanaFD{unify them into a single vertex}{make them one vertex}) \DanaFD{and}{we} get an acyclic graph $\G'$. Using the topological sorting of $\G'$, we get a satisfying assignment $a$ for $F_1$ that satisfies $a_i=a_j$. We now show that $F_2(a)=0$ and thus get a contradiction. Since $R(\I(F_2))_{i,j}=1$, there is a path $i=i_1\to i_2\to \cdots \to i_\ell=j$ from $i$ to $j$ in $\G_{\I(F_2)}$. Therefore, $F_2$ contains $F':=[x_{i_1}>x_{i_2}]\wedge [x_{i_2}>x_{i_3}]\wedge\cdots \wedge [x_{i_{\ell-1}}>x_{i_\ell}]$. Since $F_2\imply F'\imply [x_{i_1}>x_{i_\ell}]=[x_i>x_j]$ and our assignment satisfies $[a_i>a_j]=0$, we get $F_2(a)=0$.

\emph{If:} Assume $R(\I(F_2))=R(\I(F_1))$. Suppose, on the contrary, that $F_2\not=F_1$. Then, there is an assignment $a$ such that $F_2(a)=1$ and $F_1(a)=0$ (or vice versa). Since $F_1(a)=0$, $a$ is not a topological sorting of $\G_{\I(F_1)}$. Therefore, there is an edge $i\to j$ in $\G_{\I(F_1)}$ such that $a_i\le a_j$. Since $R(\I(F_2))_{i,j}$$=$$R(\I(F_1))_{i,j}$$=$$1$, there is a path from $i$ to $j$ in $\G_{\I(F_2)}$. As before, we get a contradiction.\qed\\
{\bf Proof of Lemma \ref{DI}} Since $(r,s)\in \I(G)$, we have $R(\I(G))_{r,s}$$=$$1$. On the other hand,
since there is no path from $r$ to $s$ in $\G_{\I(G)\backslash \{(r,s)\}}$, we have $R(\I(G^{r,s}))_{r,s}=0$. Therefore, $R(\I(G))\not= R(\I(G^{r,s}))$ and by Lemma~\ref{D02}, we get $G\not= G^{r,s}$. By Lemma~\ref{imm}, $G^{r,s}$ is an immediate descendant of $G$.

To show that there is no other immediate descendant, we use (the dual result of) Lemma~\ref{allim}. Note that $S(G)\backslash S(G^{r,s})=\{[x_{r}>x_{s}]\}$ and thus, by Lemma~\ref{allim}, it is sufficient to prove that
$G=G':=\wedge_{(i,j)\in J} [x_i>x_j]$, where $J=\{(i,j)\in \I(G)\ |\ $ there is no path from $i$ to $j$ in $\G_{\I(G)\backslash\{(i,j)\}}\}$. To prove it, we show $R(\I(G'))=R(J)$, and then the result follows from Lemma~\ref{D02}.

If $R(J)_{i,j}=1$, then $R(\I(G'))_{i,j}=1$ since $\G_{\I(G')}$ is a subgraph of $\G_{\I(G)}$. If $R(\I(G'))_{i,j}=1$, then there is a path from $i$ to $j$ in $\G_{\I(G')}$, and therefore, there is a path from $i$ to $j$ in $\G_{\I(G)}$, and thus $R(\I(G))_{i,j}=1$. Since $R(\I(G))_{i,j}=1$, there is a path $p$ from $i$ to $j$ in $\G_{\I(G)}$. Let $(r,s)\not \in \I(G)\backslash J$. Then, $(r,s)\in \I(G)$ and there is a path (other than $r\to s$) $r\to v_1\to v_2\to \cdots\to v_\ell=s$ in $\G_{\I(G)}$. We now show that there is a path from $i$ to $j$ in $\G_{\I(G)\backslash \{(r,s)\}}$. This is true because if the path $p$ (in $\G_{\I(G)}$) contains the edge $r\to s$, then we can replace this edge with the path $r\to v_1\to v_2\to \cdots\to v_\ell=s$ and get a new path from $i$ to $j$ in $\G_{\I(G)\backslash \{(r,s)\}}$. Therefore, $R(\I(G)\backslash \{(r,s)\})_{i,j}=1$. By repeating this on the other edges in $\I(G)\backslash J$, we get $R(J)_{i,j}=1$.\qed\\
{\bf Proof of Theorem~\ref{lasttheo}}:
\emph{If:} Let ${\cal A}$ be an algorithm that for an input $\G$, which is a directed graph, enumerates all the maximal acyclic subgraphs in polynomial time (poly$(N(\G),|V|,|E|)$). The first step of SPEX (in Figure~\ref{Alg1}) finds all the immediate descendants of $G_{\max}$. By Lemma~\ref{IDG}, this is equivalent to enumerating all the maximal acyclic subgraphs of $\G_I$. This can be done by ${\cal A}$ in time $poly(N(\G_I),n,|I|)$. For every immediate descendant $G'$ of $G_{\max}=0$, any topological sorting of $G'$ is a witness for $G'$ and $G$.
Once SPEX calls Learn on one of the immediate descendants of $G_{\max}$, the algorithm proceeds as in the acyclic case. This algorithm runs in time $poly(N(\G_I),n,|I|)$ time and asks at most $N(G_I)+|I|$ membership queries. By Lemma~\ref{loweri}, the algorithm runs in time $poly(\OPT(\FF^I_\wedge),n,|I|)$ and asks at most $\OPT(\FF^I_\wedge)+|I|$ queries.\\
\begin{figure}[t]
\centering
\includegraphics[trim = 0 1cm 0 1cm,width=0.7\textwidth,height=4.5cm]{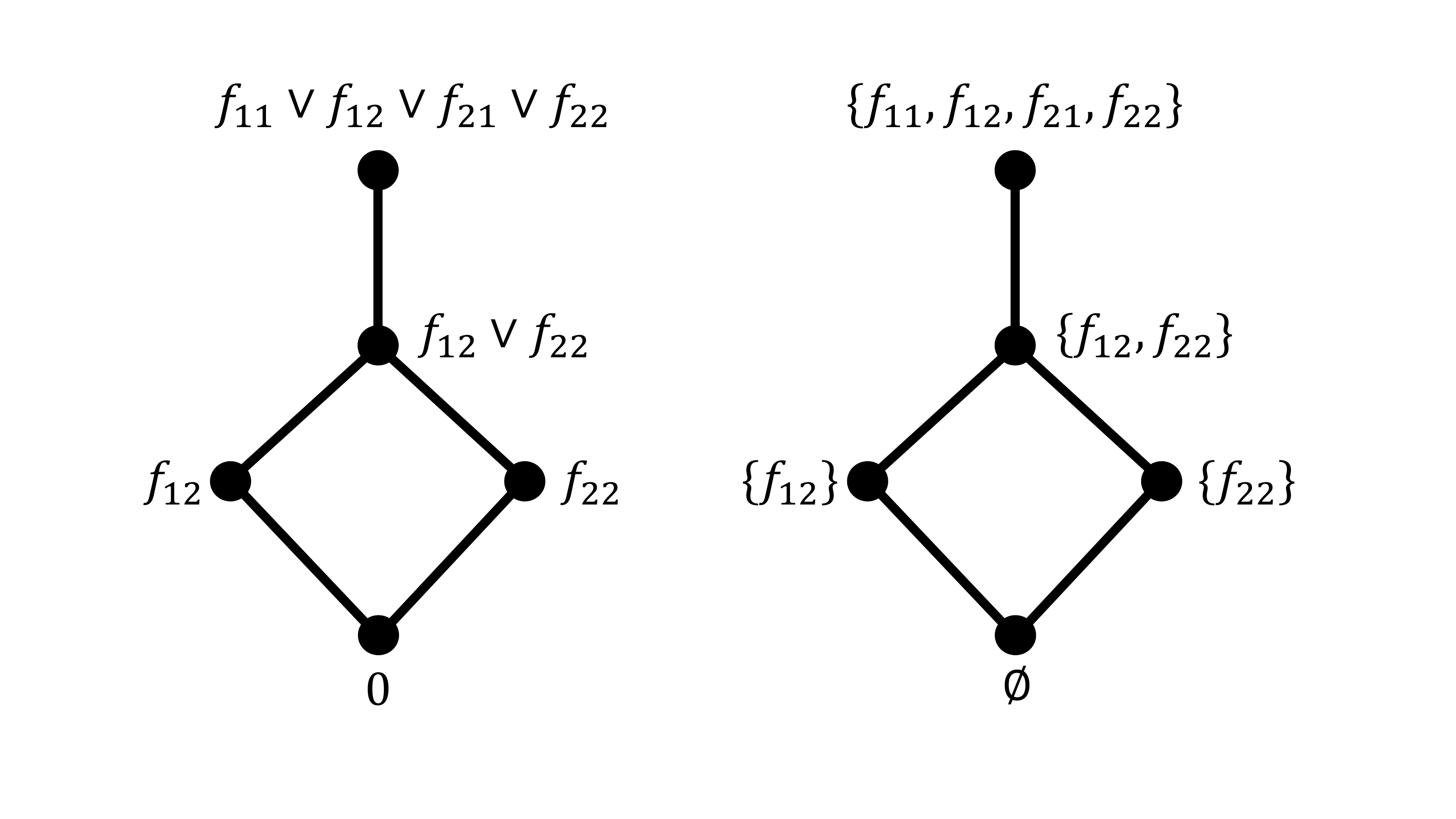}
\caption{Left: the Hasse diagram of $\Ray_2^2$. Right: the corresponding $S(G)$ sets.}
\label{HasseRay22}
\end{figure}
\begin{figure}[t]
\centering
\includegraphics[trim = 0 1cm 0 1cm,width=0.9\textwidth,height=5.5cm]{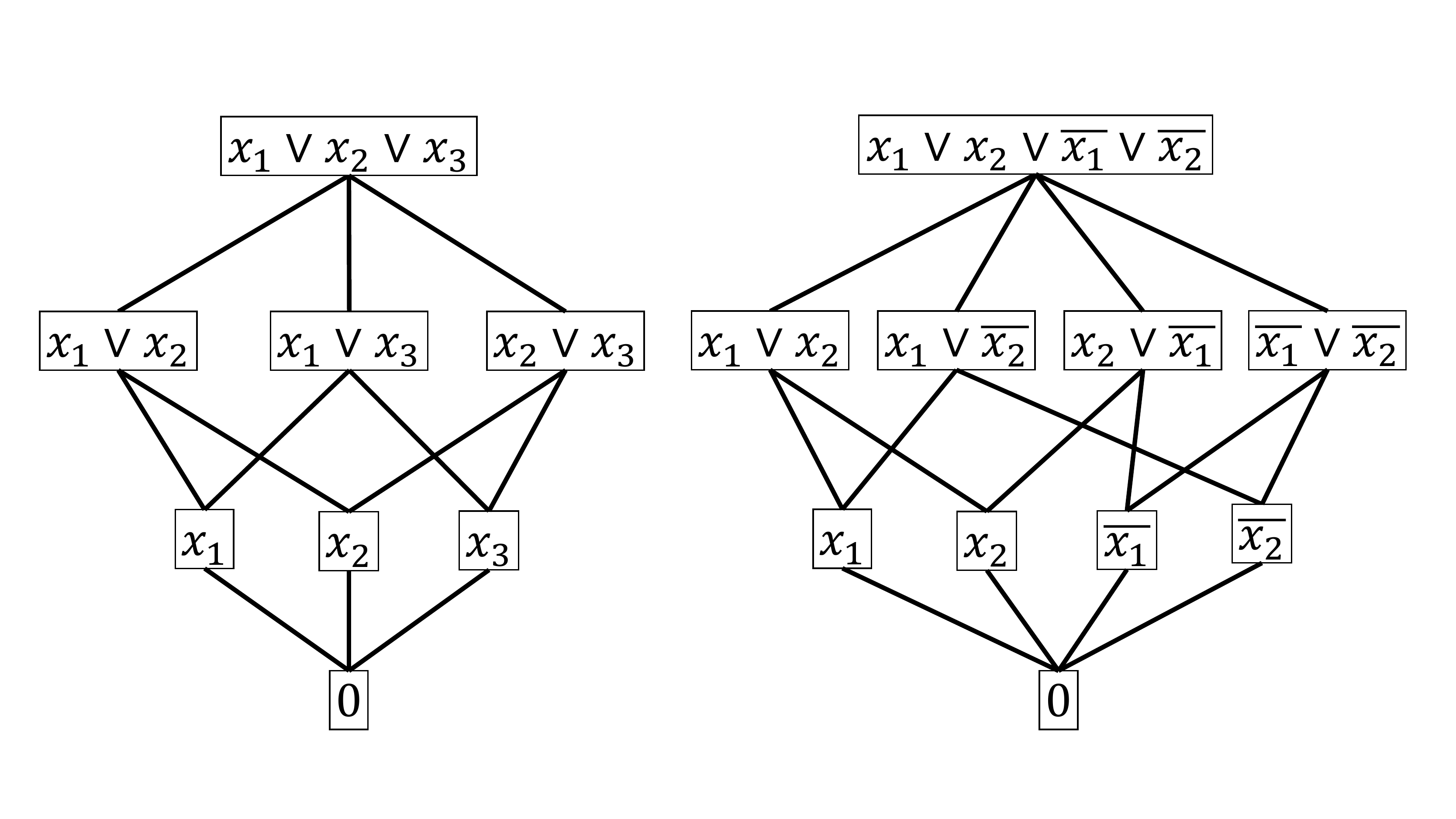}
\caption{Hasse diagram of terms and monotone terms.}
\label{HasseClause32}
\end{figure}
\emph{Only if:} Let ${\cal B}$ be a learning algorithm that runs in poly$(\OPT(\FF^I_\wedge),n,|I|)$. By the above argument:
\begin{eqnarray}\label{eq34}
\OPT(\FF^I_\wedge)\le N(\G_I)+|I|.
\end{eqnarray}
Let $\G=([n],E)$ be any directed graph. We run the learning algorithm with the target $F_I$ where $I=E$. For any membership query asked by the algorithm, we answer $0$ until the algorithm outputs the hypothesis $G_{\max}=0$. Suppose $A$ is the set of all membership queries that are asked by the algorithm. We now claim that:
\begin{enumerate}[nosep,nolistsep]
\item $|A|=poly(N(\G), n,|E|)$.
\item If $G'$ is a maximal acyclic subgraph of $G$, then there is an assignment $a\in A$ such that $E(G')=\{(i,j)\in E\ | \ a_i>a_j\}$, where $E(G')$ is the set of edges of~$G'$.
\end{enumerate}
Bullet 1 follows since ${\cal B}$ runs in time $poly(\OPT(\FF^I_\wedge),n,|I|)$ and by (\ref{eq34}) this is $poly(N(\G), n,|E|)$. So the number of membership queries cannot be more than $poly(N(\G), n,|E|)$ time.

We now prove bullet 2. There is an assignment $a\in A$ that satisfies $F_{E(G')}(a)=1$, because otherwise the algorithm cannot distinguish between $F_{E(G')}$ and $G_{\max}$, which violates the correctness of the algorithm. Now, since $F_{E(G')}(a)=1$, we must have $E(G')\subseteq \{(i,j)\in E\ | \ a_i>a_j\}$. Since $E(G')$ is maximal (adding another edge will create a cycle), we get $E(G')= \{(i,j)\in E\ | \ a_i>a_j\}$.

The algorithm that enumerates all the maximal acyclic subgraphs of $\G(V,E)$ continues to run as follows: for each $a\in A$, it defines $E_a:=\{(i,j)\in E\ | \ a_i>a_j\}$. If $G_a:=([n],E_a)$ is a maximal cyclic subgraph, then it lists $G_a$.
Testing whether $G_a:=([n],E_a)$ is maximal can be done in polynomial time (e.g., by checking edge-by-edge in $E$). It is easy to verify that the algorithm runs in poly$(N(\G),|V|,|E|)$ time.\qed

\begin{figure}
\centering
\includegraphics[trim = 0 1cm 0 1cm,width=0.9\textwidth,height=8cm]{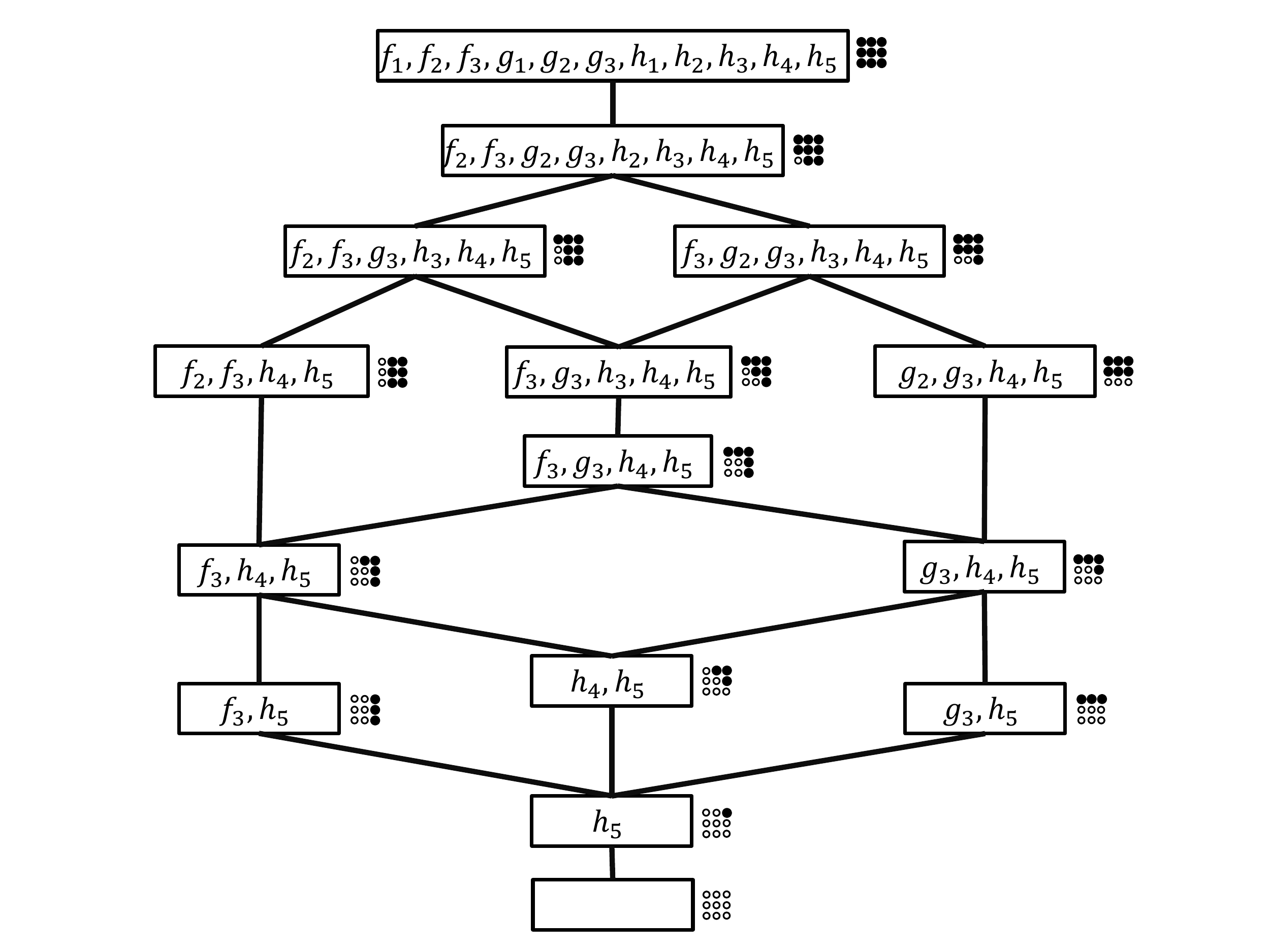}
\caption{Hasse diagram of $\FF$$=$$\{f_1,f_2,f_3,g_1,g_2,g_3,h_1,\ldots,h_5\}$ whose functions are $\{1,2,3\}$$\times$$\{1,2,3\}$$\to$$\{0,1\}$ where $f_i(x_1,y_1)$$=$$[x_1\ge i]$, $g_i(x_1,x_2)$$=$$[x_2\ge i]$ and $h_i(x_1,x_2)$$=$$[x_1+x_2\ge i+1]$.}
\label{RAY23E}
\end{figure}

\begin{figure}
\centering
\includegraphics[trim = 0 1cm 0 1cm,width=0.9\textwidth,height=9cm]{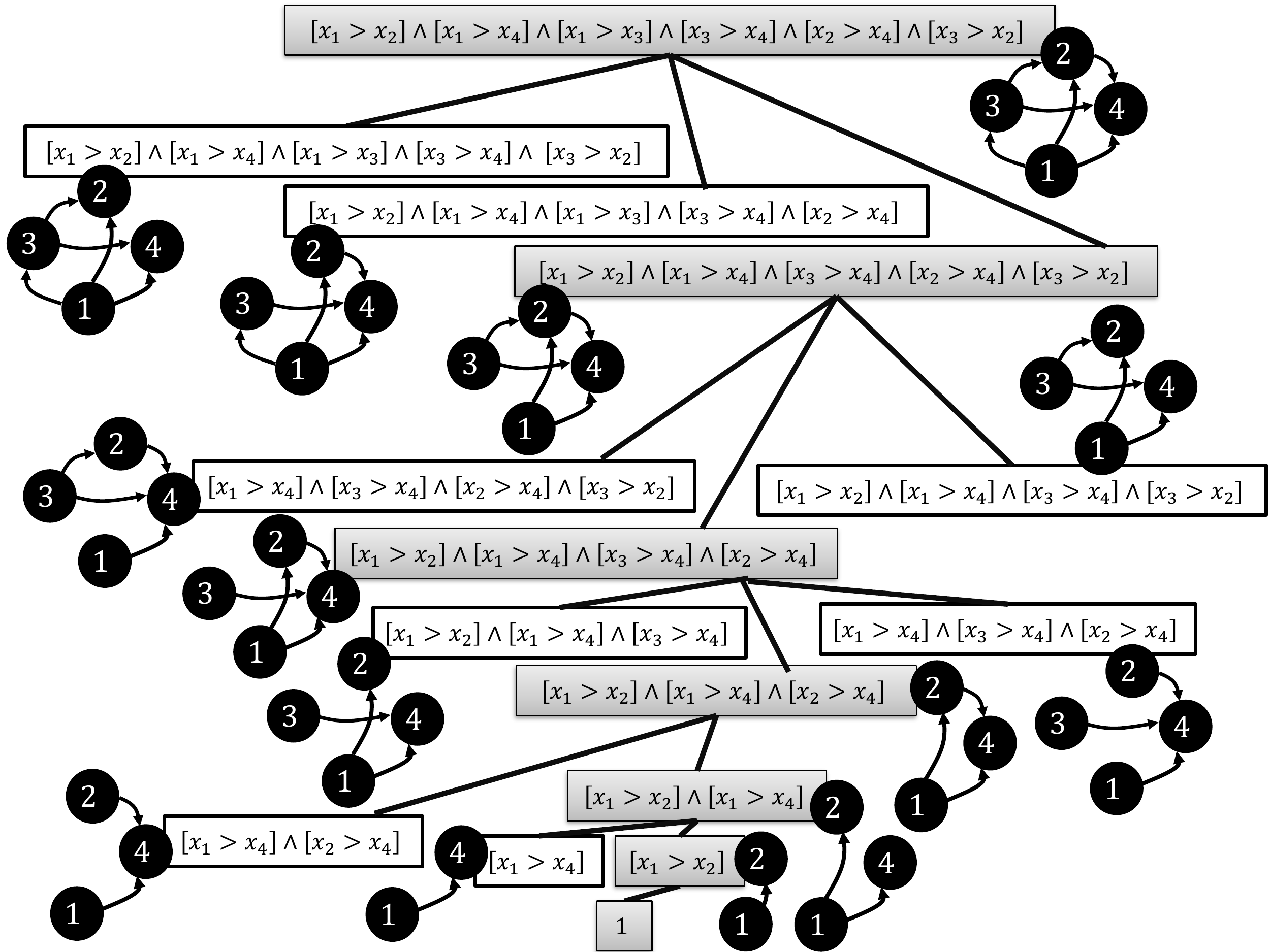}
\caption{A path from $G_{\max}$ to $G_{\min}$ in the inequality predicate diagram. Here $\I(G_{\max})=\{(1,2),(1,4),(1,3),(3,4),(2,4),(3,2)\}$}
\label{Graph}
\end{figure}
\section{Additional Figures}
Here, we provide Figures~\ref{HasseRay22}, \ref{HasseClause32}, \ref{RAY23E}, and \ref{Graph}.
\end{document}